\documentclass[11pt]{article}
\usepackage{fullpage}

\usepackage[utf8]{inputenc} 
\usepackage[T1]{fontenc}    
\usepackage{hyperref}       
\usepackage{url}            
\usepackage{booktabs}       
\usepackage{amsfonts}       
\usepackage{nicefrac}       
\usepackage{microtype}      
\usepackage{subcaption}
\usepackage{latexsym}
\usepackage[cmex10]{amsmath}
\usepackage{bbm}
\usepackage{amssymb}
\usepackage{graphicx}
\usepackage[ruled]{algorithm}
\usepackage{algpseudocode}
\usepackage{array}
\usepackage{flexisym}
\usepackage[mathscr]{euscript}
\usepackage{mathtools}
\usepackage{amsthm}
\usepackage{mathptmx}
\DeclareGraphicsExtensions{.density}

\let\originalleft\left
\let\originalright\right
\renewcommand{\left}{\mathopen{}\mathclose\bgroup\originalleft}
\renewcommand{\right}{\aftergroup\egroup\originalright}

\newtheorem{theorem}{\it{\bf Theorem}}
\newtheorem{lemma}{\it{\bf Lemma}}

\newcommand{\nn}{\nonumber}
\newcommand{\defi}{\triangleq}
\DeclareMathOperator{\E}{\mathbb{E}}
\newcommand{\Ew}{\E_{\bw}}
\newcommand{\Ek}{\E_k}

\newcommand\given[1][]{\:#1\Bigg\vert\:}
\newcommand{\LVert }{\left \Vert}
\newcommand{\RVert }{\right \Vert}
\newcommand{\sql }{\left[}
\newcommand{\sqr }{\right]}
\newcommand{\crl}{\left\{ }
\newcommand{\crr}{\right\} }
\newcommand{\prl}{\left( }
\newcommand{\prr}{\right) }
\newcommand{\eps}{\varepsilon}

\DeclareMathOperator*{\argmin}{\arg\!\min\;}

\newcommand{\1}{\mathbbm{1}}
\newcommand{\mA}{\mathbf A}
\newcommand{\mE}{\mathbf E}
\newcommand{\mI}{\mathbf I}
\newcommand{\bQ}{\mathbf{Q}}
\newcommand{\bv}{\mathbf v}

\newcommand{\bW}{\mathbf W}
\newcommand{\bw}{\mathbf{w}}

\newcommand{\Tr}{\mathrm{Tr}}
\newcommand{\diag}{\mathrm{diag}}

\newcommand{\stn}{\mathbf{s}_{t,n}}
\newcommand{\Atn}{\mathbf{A}_{t,n}}

\newcommand{\DF}{\nabla F}
\newcommand{\DFw}{\DF \prl \bw \prr}
\newcommand{\DFk}{\DF \prl \bw_k \prr}
\newcommand{\DFak}{\DF \prl \bw_k \prr}
\newcommand{\DFok}{\DF \prl \bwo_k \prr}

\newcommand{\mMw}{\mathbf{M} \prl \bw \prr}

\newcommand{\mMwkOpt}{\mathbf{M} \prl \bw_k^{\mathrm{o}} \prr}

\newcommand{\cmin}{c_{\mathrm{min}}}
\newcommand{\Fstar}{F\prl \bw_* \prr}

\newcommand{\gwt}{\mathbf{g} \prl \bw, \theta \prr}
\newcommand{\giwt}{\mathbf{g}_i \prl \bw, \theta \prr}

\newcommand{\Gw}{\mathbf{G} \prl \bw \prr}
\newcommand{\Gwk}{\mathbf{G} \prl \bw_k \prr}
\newcommand{\Gwn}{\mathbf{G} \prl \bw, n \prr}
\newcommand{\Gtw}{\mathbf{G}_t \prl \bw \prr}
\newcommand{\Gtwn}{\mathbf{G}_t\prl \bw, n \prr}
\newcommand{\Gtwnt}{\mathbf{G}_t\prl \bw, n_t \prr}
\newcommand{\Giwn}{\mathbf{G}_i \prl \bw, n \prr}
\newcommand{\Gjwn}{\mathbf{G}_j \prl \bw, n \prr}
\newcommand{\Gsw}{\mathbf{G}^{*} \prl \bw \prr}
\newcommand{\GtoknOpt}{\mathbf{G}_{t}^{\mathrm{o}} \prl \bwok,  \nOpt \prr}
\newcommand{\Gok}{\mathbf{G}^{\mathrm{o}} \prl \bwok \prr}
\newcommand{\Gtok}{\mathbf{G}_{t}^{\mathrm{o}} \prl \bwok \prr}
\newcommand{\GOnewn}{\mathbf{G}_{1} \prl \bw, n\prr}
\newcommand{\Ggammawn}{\mathbf{G}_{\gammatn} \prl \bw, n \prr}
\newcommand{\GOnewOpt}{\mathbf{G}_{1} \prl \bw, \nOpt \prr}
\newcommand{\GtwOpt}{\mathbf{G}_{t} \prl \bw, \nOpt \prr}
\newcommand{\GTwOpt}{\mathbf{G}_{T} \prl \bw, \nOpt \prr}

\newcommand{\Uwt}{\mathbf{U} \prl \bw, \theta \prr}
\newcommand{\Uiwt}{\mathbf{U}_i \prl \bw, \theta \prr}
\newcommand{\vUwti}{\mathbf{U} \prl \bw, \theta \prr_i}

\newcommand{\MBOn}{\mathrm{MBO} \prl n, B, \bw \prr}

\newcommand{\MBOopt}{\mathrm{MBO} \prl \nOpt, B, \bw \prr}
\newcommand{\MBOoptk}{\mathrm{MBO} \prl \nOptk, B, \bwok \prr}
\newcommand{\MBOnt}{\mathrm{MBO} \prl n_t, B, \bw \prr}

\newcommand{\sigman}{\sigma_n^2}

\newcommand{\sigmaOpt}{\sigma_{*}^2}
\newcommand{\sigmaOptk}{\sigma_{*,k}^2}
\newcommand{\sigmant}{\sigma_{n_t}^2}
\newcommand{\nOpt}{n_{*}}
\newcommand{\nOptk}{n_{*,k}}

\newcommand{\bmtn}{\mathbf{m}_t \prl n \prr}
\newcommand{\bqtn}{\mathbf{q}_t\prl n \prr}

\newcommand{\bCtn}{\mathbf{Cov}_t \prl n \prr}

\newcommand{\Vtn}{V_t \prl n\prr}
\newcommand{\VtnOpt}{V_t \prl \nOpt \prr}
\newcommand{\Vsn}{V_s^{\prime} \prl \nOpt \prr}

\newcommand{\gammatn}{\gamma_t \prl n \prr}

\newcommand{\gammatnOpt}{\gamma_t \prl \nOpt \prr}
\newcommand{\gammaTn}{\gamma_T \prl n \prr}

\newcommand{\ZTw}{Z_T\prl \bw \prr}
\newcommand{\ZTwk}{Z_T\prl \bw_k \prr}
\newcommand{\OTk}{O_{T,k}}

\newcommand{\Cone}{\mathrm{C}_1 \prl \bw \prr}
\newcommand{\Conek}{\mathrm{C}_{1,k} \prl \bw \prr}
\newcommand{\Ctwo}{\mathrm{C}_{2}}
\newcommand{\Ctwok}{\mathrm{C}_{2,k}}
\newcommand{\Sw}{S \prl \bw \prr}
\newcommand{\Swk}{S \prl \bw_k \prr}
\newcommand{\Swo}{S \prl \bw_k^{\mathrm{o}} \prr}

\newcommand{\bwo}{\bw^{\mathrm{o}}}
\newcommand{\bwok}{\bw^{\mathrm{o}}_k}
\newcommand{\JAk}{J_{k,\eta}}
\newcommand{\JAkp}{J_{k+1,\eta}}
\newcommand{\JOk}{J_{k,\eta}^{\mathrm{o}}}
\newcommand{\JOkp}{J_{k+1,\eta}^{\mathrm{o}}}
\newcommand{\taoAlg}{\tau_k \prl \eta_k \prr}
\newcommand{\taoOpt}{\tau_k^{\mathrm{o}} \prl \eta_k \prr}

\begin{document}

\title{EE-Grad: Exploration and Exploitation for Cost-Efficient Mini-Batch SGD\thanks{This work was supported in part by Systems on Nanoscale Information fabriCs (SONIC), one of the six SRC STARnet Centers, sponsored by MARCO and DARPA, and in part by the Center for Science of Information (CSoI), an NSF Science and Technology Center, under
grant agreement CCF-0939370.}}

\author{Mehmet A.~Donmez\thanks{donmez2@illinois.edu} \and Maxim Raginsky\thanks{maxim@illinois.edu} \and Andrew C.~Singer\thanks{acsinger@illinois.edu}}

\maketitle

\begin{abstract}
We present a generic framework for trading off fidelity and cost in computing stochastic gradients when the costs of acquiring stochastic gradients of different quality are not known a priori. We consider a mini-batch oracle that distributes a limited query budget over a number of stochastic gradients and aggregates them to estimate the true gradient. Since the optimal mini-batch size depends on the unknown cost-fidelity function,
we propose an algorithm, {\it EE-Grad}, that sequentially explores the performance of
mini-batch oracles and exploits the accumulated knowledge to estimate the one
achieving the best performance in terms of cost-efficiency.
We provide performance guarantees for EE-Grad with respect to the optimal mini-batch oracle,
and illustrate these results in the case of strongly convex objectives.
We also provide a simple numerical example that corroborates our theoretical findings.
\end{abstract}

\section{Introduction}
Stochastic gradient methods are widely used to solve large-scale optimization problems
in machine learning.
Given a differentiable objective function $F: \mathbb{R}^d \rightarrow \mathbb{R}$ with
a gradient $\DF$, a stochastic gradient descent (SGD) algorithm chooses an initial iterate
$\bw_1\in\mathbb{R}^d$, and, on each iteration $k=1,\ldots,K$,
it uses a noisy gradient $ \Gwk$ instead of $\DFk$ to set the next iterate as
$
	\bw_{k+1} = \bw_k - \eta_k \Gwk,
$
where $\eta_k>0$ is a step size.
The overall performance of stochastic gradient methods is controlled by the noise in $\Gwk$
with respect to $\DFk$ \cite{BottouCN2016}.
Often, noisy gradients with large variances lead to slower convergence and degraded performance \cite{WangCSX2013}.

Mini-batch stochastic gradient methods, as well as their distributed or parallelized variants, have been proposed to tackle some of these issues \cite{LiZCS2014,ZinkevichWSL2010}.
Recently, federated learning \cite{McMahanMRHA2016} has been proposed as a decentralized
optimization framework, where SGD runs on a large dataset distributed across a number of
devices performing local model updates and sending them to a centralized server that aggregates
them, under privacy and communication constraints.
In typical resource- and budget-constrained applications, as the mini-batch size increases, the cost available to be allocated
to each single stochastic gradient in the mini-batch decreases, so that its quality degrades, i.e.,
its noise variance increases.
A common approach is to focus on the tradeoff between the rate of convergence
and the computational complexity of stochastic gradient methods, where the dependence
of the noise variance on the cost allocated to stochastic gradients is often omitted.

In this paper, we propose an alternative framework and consider the tradeoff between {\it fidelity}
and {\it cost} of computing a stochastic gradient.
In particular, we model a noisy gradient as an unbiased estimate of the true gradient,
where the noise variance depends on the incurred cost, and this dependence is formalized
through a {\it cost-fidelity} function.
We focus on mini-batch oracles, where each mini-batch oracle distributes
a limited budget across a mini-batch of stochastic gradients and aggregates them to
form a final gradient estimate.
We assume that the aggregation operation also incurs a cost from the budget, as does each of the noisy
gradients in the mini-batch.
The optimal mini-batch size in minimizing the noise variance depends on the underlying
cost-fidelity function.

We focus on determining the optimal mini-batch oracle in terms of the cost-fidelity tradeoff
when the cost-fidelity function is unknown.
In particular, we propose and analyze EE-Grad: an algorithm that, on each iteration, performs sequential trials over different mini-batch oracles to {\it explore} the performance
of each mini-batch oracle with high precision and {\it exploit} the current knowledge to focus on
the one that seems to provide the best performance, i.e, the smallest noise variance.
We demonstrate that the proposed algorithm performs almost as well as the optimal mini-batch oracle on each iteration in expectation.
We apply this result to the case of strongly convex objectives, and prove performance
guarantees in terms of the rate of convergence.
We finally provide a numerical example to illustrate our theoretical results.

\section{Cost-Fidelity Tradeoff and Mini-Batch Stochastic Gradient Oracles}
\label{sec:problemDescription}
Suppose that, on each iteration, a stochastic gradient $\gwt$
and the gradient $\DF \prl \bw \prr$ are related as
\begin{align}
	\gwt = \DFw  + \Uwt \label{eq:model}
\end{align}
where $\Uwt$ is a zero-mean perturbation with a positive definite
and diagonal covariance matrix $\theta^{-1} \mMw$ for $\theta>0$.
That is,
\begin{align*}
	\Ew \sql \Uwt \sqr = 0, \;\;\;\;
	\Ew \sql \Uwt \Uwt^T \sqr = \theta^{-1} \mMw,
\end{align*}
where $\Ew \sql \cdot \sqr$ is the conditional expectation given $\bw$.
Here, $\theta$ is the {\it fidelity} of the stochastic gradient $\gwt$.
We assume that $i$th element of $\Uwt$ is sub-Gaussian
with the parameter $\theta^{-1} \mMw_{i,i}$, i.e.,
\begin{align}
	\Ew \sql e^{ \lambda \vUwti } \sqr 
	\leq e^ { \lambda^2 \mMw_{i,i} / 2\theta },
	\;\;\; \forall \lambda \in \mathbb{R},
	\label{eq:subGaussian}
\end{align}
for $i \in \sql d \sqr$\footnote{For any positive integer $N$, $\sql N \sqr \defi \crl 1,\ldots,N\crr$.}.
A mini-batch stochastic gradient is computed by averaging $n$ independent
noisy gradients 
$
	\giwt = \DFw + \Uiwt,
$
 $i\in \sql n \sqr$, each with fidelity $\theta$:
\begin{align}
	\Gw = \frac{1}{n} \sum_{i=1}^n \giwt, \label{eq:minibatch}
\end{align}
which has the covariance matrix $ \mMw /n\theta$,
and satisfies $\Ew \sql \LVert \DFw - \Gw \RVert _2^2 \sqr = \Sw/n\theta$,
where $\Sw = \Tr \prl \mMw \prr$ is the trace of the covariance matrix.

A stochastic gradient $\gwt$ with fidelity $\theta>0$ incurs a cost $C \prl \theta \prr$,
which is a strictly increasing function of $\theta$
with $\lim_{\theta\rightarrow 0} C \prl \theta  \prr = \cmin \geq 0$.
We assume that the cost function $C\prl\theta \prr$ is unknown.
There is also an aggregation cost $D \prl n \prr$ to perform the averaging operation,
where $D \prl n \prr$ is increasing with $D \prl 1\prr = 0$.
Hence, given a budget $B>0$, the maximum feasible mini-batch size is
\begin{equation*}
	N = \max \crl n \in \mathbb{Z}_+ \given B > n\cmin + D \prl n \prr \crr.
\end{equation*}
Here, we define, for each $n\in\sql N\sqr$, a mini-batch oracle $\MBOn$ that computes a
mini-batch stochastic gradient $\Gwn$ as in \eqref{eq:minibatch} using the 
fidelity
\begin{align*}
	\theta_n \defi C^{-1} \prl \frac{B - D \prl n \prr}{n} \prr.
\end{align*}
That is, each individual stochastic gradient in the mini-batch is allocated 
$\prl B - D \prl n \prr \prr / n$ in cost.
Therefore, the covariance matrix of $\Gwn$ is $\sigman \mMw$, where
\begin{align*}
	\sigman \defi \frac{1}{n\theta_n}
\end{align*}
is unknown, since the cost function $C\prl \theta \prr$ is assumed unknown.
Note that, given $\DF \prl \bw \prr$, the concentration of $\Gwn$ around $\DF \prl \bw \prr$ is completely governed by $\sigman$ for each $n\in\sql N\sqr$.
The optimal mini-batch size in terms of the cost-fidelity tradeoff is given by 
\begin{align*}
	\nOpt \defi \argmin_{n=1,\ldots,N} \sigman,
\end{align*}
and $\sigmaOpt \defi \sigma_{\nOpt}$.
In particular, we define the suboptimality gap of each mini-batch oracle $\MBOn$
\begin{equation*}
	\Delta_n \defi \sigman - \sigmaOpt \geq 0.
\end{equation*}

Since the cost function is unknown, the optimal mini-batch size $\nOpt$ and
$\sigmaOpt$, and hence the optimal mini-batch oracle $\MBOopt$, are unknown.
In the next section, we propose an algorithm that attempts to {\it learn}
the optimal mini-batch oracle over sequential trials in the sense that
its noise variance is almost as small as the optimal mini-batch oracle on each iteration.

\section{The EE-Grad Algorithm}
\label{sec:Algorithm}
In this section, we present EE-Grad: an algorithm that, on each iteration of the SGD, aggregates stochastic gradients computed over sequential trials, where at each trial it estimates the optimal mini-batch size and uses the available per-round budget to query the corresponding mini-batch oracle.
EE-Grad constructs a high confidence bound on the variance estimate
of each mini-batch oracle by exploiting the sub-Gaussian assumption on the noisy gradients.
We demonstrate that, in expectation, the algorithm performs almost as well as the optimal mini-batch oracle at each iteration.

On each SGD iteration, EE-Grad performs the following $T$-round procedure.
On round $t=1,\ldots,T$, it picks a mini-batch size $n_t \in \sql N\sqr$ based
on a strategy introduced later in this section, and uses the per-round budget $B$
to query the mini-batch oracle $\MBOnt$. The oracle returns $\Gtw=\Gtwnt$, an unbiased estimate of $\DFw$, with covariance matrix $\sigmant \mMw$. After $T$ rounds, the algorithm outputs the stochastic gradient
$
	\Gw = (1/T)\sum_{t=1}^T \Gtw.
$

We denote the number of rounds the algorithm picks $\MBOn$ up to round $t$
as $\gammatn$, index its outputs as $\GOnewn,\ldots,\Ggammawn$,
and write its sample mean and sample covariance matrix as
\begin{align*}
	\bmtn 
	= \frac{1}{\gammatn} \sum_{i=1}^{\gammatn} \Giwn, \;\;\;\;
	\bCtn
	= \frac{1}{\gammatn - 1} \sum_{i=1}^{\gammatn} \prl \Giwn - \bmtn \prr \prl \Giwn - \bmtn \prr^T,
\end{align*}
respectively, for $n \in \sql N \sqr$.
The algorithm computes the trace of the sample covariance matrix, denoted by
$\Vtn = \Tr\prl \bCtn \prr$ for each $n\in\sql N\sqr$.
Note that $\Ew \sql \Vtn \sqr = \sigman \Sw$, which implies that for each $\MBOn$,
the trace of its sample covariance matrix is an unbiased estimate of $\sigman \Sw$.

We emphasize that this framework is similar to the stochastic multi-armed bandit setup that involves an exploration/exploitation tradeoff when picking different arms over
sequential trials \cite{BubeckC2012}.
In particular, algorithms that exploit the available knowledge on the current best arm and explore the other arms to estimate the actual best arm with higher precision have been shown to yield satisfactory performance \cite{BubeckC2012,Auer2002}.
We adopt a similar approach here, and propose an algorithm that simultaneously performs
exploration and exploitation.
More precisely, EE-Grad first initializes by picking each mini-batch oracle exactly twice, so that
$\gammatn = 2$ for each $n\in \sql N\sqr$ at trial $t=2N$,
and then picks the mini-batch oracle at trial $t = 2N+1,\ldots,T$ according to
\begin{align}
	n_t \in \argmin\limits_{n=1,\ldots,N} \sql \Vtn - f \prl \frac{\alpha\ln\prl t\prr}{\gammatn-1} \prr \sqr,
	\label{eq:pickOracle}
\end{align}
for some $\alpha>2$, where
\begin{align}
	f \prl x \prr \defi \beta P \sqrt{\frac{xd}{c}} \max\prl 1, \sqrt{\frac{x}{cd}}\prr,
	\label{eq:funcf}
\end{align}
and $c>0$ is a universal constant that comes from the use of Hanson-Wright inequality, 
as detailed in the proof of Theorem~\ref{thm:main}.
Here we assume that $\beta$ and $P$ are known constants such that
$ \sigman \leq \beta$ for each $n\in\sql N\sqr$, and $\Sw \leq P$.
This algorithm constructs an upper confidence bound on the trace of the sample covariance matrix of each mini-batch oracle, and picks the one with the best estimate. The overall scheme, presented below as Algorithm~\ref{alg:alg1}, will be analyzed using techniques similar to the ones used in UCB strategies \cite{Auer2002, Robbins1952, AuerO2010}, as explained in the proof of Theorem~\ref{thm:main}.

\begin{algorithm}[H]
\caption{EE-Grad}
\label{alg:alg1}
\begin{algorithmic}[0]
\State {\bf Input:} $N,T>1$, per-round budget $B$.
\State {\bf Initialization:}  Pick each mini-batch oracle twice in the first $2N$ rounds.
\For{$t=2N+1:T$}
\State Compute $\Vtn$ for each $n\in\sql N \sqr$, and pick a mini-batch size $n_t$ based on \eqref{eq:pickOracle}.
\State Distribute the budget $B$ to $\MBOnt$, which reveals $\Gtwn$, and set $\Gtw = \Gtwn$.
\EndFor
\State Compute the final gradient estimate as
$
	\Gw = (1/T) \sum_{t=1}^T \Gtw,
$
\end{algorithmic}
\end{algorithm}

\section{EE-Grad Performance Guarantees}
\label{sec:mainResults}
In this section, we investigate the performance of EE-Grad.
In particular, we prove an upper bound on its noise variance,
and compare it to the noise variance achieved by the optimal mini-batch oracle:
\begin{theorem}
\label{thm:main}
On each iteration, the stochastic gradient computed by EE-Grad satisfies
\begin{align*}
	\Ew \sql \LVert \Gw - \DFw  \RVert_2^2 \sqr \leq  \ZTw \Sw,
\end{align*}
where
\begin{align*}
	\ZTw =  \frac{\sigmaOpt}{T} + \prl \frac{\ln T }{T^2} \prr  \Cone + \prl \frac{1}{T^2}\prr \Ctwo,
\end{align*}
and 
\begin{align*}
	\Cone
	\defi\sum_{n:\Delta_n>0} \frac{\alpha \Delta_n }{\phi \prl \Delta_n \Sw/2\prr},
	\;\;\;\;\;
	\Ctwo 
	\defi \prl \sum_{n=1}^{N} \Delta_n \prr \frac{2\prl \alpha-1\prr}{\alpha-2},
	\;\;\;\;\;
	\phi \prl \eps \prr 
	&\defi \frac{ c \eps}{\beta P}\min\prl 1,\frac{\eps/d}{\beta P}\prr.
\end{align*}
 Also, the stochastic gradient $\Gsw$ computed by the optimal mini-batch oracle satisfies
\begin{align*}
	\Ew \sql \LVert \Gsw - \DFw \RVert_2^2 \sqr= \frac{\sigmaOpt}{T} \Sw.
\end{align*}
 
\end{theorem}
\begin{proof}
We prove this theorem in several steps.
We first analyze the difference between the noise variance of the stochastic gradient generated by EE-Grad and
that of the optimal mini-batch oracle.
We next show that this quantity is related to the pseudo-regret term that appears in stochastic multi-armed bandit problems, where UCB-type strategies are used to achieve upper bounds on
the pseudo-regret by leveraging concentration inequalities.
We present a similar formulation to analyze the behavior of the proposed algorithm with respect
to the optimal mini-batch oracle.
To prove the upper bound, we first demonstrate that the trace of the sample covariance matrix for each mini-batch oracle, which is used to pick a oracle on each trial in \eqref{eq:pickOracle},
can be written as a quadratic form of independent sub-Gaussian random variables. We combine this observation with the Hanson-Wright inequality \cite{HansonW1971} to prove a high probability tail bound on the estimate of the optimal mini-batch size.
This result also is the derivation of the rule in \eqref{eq:pickOracle}.
Based on these results, we prove a pseudo-regret bound and
connect this bound to the noise variance achieved by EE-Grad.

Note that, on each iteration, the stochastic gradient of the optimal mini-batch oracle
after $T$ rounds is
\begin{align*}
	\Gsw \defi \frac{1}{T} \sum_{t=1}^T \GtwOpt,
\end{align*}
where $\GOnewOpt,\ldots,\GTwOpt$ are independent.
We observe that
\begin{align}
	&\Ew \sql \LVert \Gw- \DFw \RVert_2^2 \sqr - \Ew \sql \LVert \Gsw - \DFw \RVert_2^2 \sqr \nn\\
	&= \frac{1}{T^2} \prl \sum_{t=1}^T \Ew \sql \LVert \Gtw - \DFw \RVert_2^2 \sqr
	- \sum_{t=1}^T \Ew \sql \LVert \GtwOpt - \DFw \RVert_2^2 \sqr \prr \nn\\
	& =  \frac{1}{T^2} \prl  \sum_{t=1}^T \Ew \sql \LVert \Gtw - \DFw \RVert_2^2 \sqr 
	- T\sigmaOpt \Sw \prr, \label{eq:pluginOptimal}
\end{align}
where in \eqref{eq:pluginOptimal} we used
$\Ew \sql \LVert \GtwOpt - \DFw \RVert_2^2 \sqr = \sigmaOpt \Sw$ for each $t \in \sql T \sqr$.
We next observe that
\begin{align}
	 \Ew \sql \LVert \Gtw- \DFw \RVert_2^2 \sqr
	 &=  \Ew \sql \Ew \sql \LVert \Gtw - \DFw \RVert_2^2 \given n_t \sqr \sqr\nn\\
	 &=  \Ew \sql \sigmant \sqr \Sw,  \label{eq:nestedExpectation}
\end{align}
where in \eqref{eq:nestedExpectation} the expectation is with respect to the randomness in $n_t$.
In particular, we can write
\begin{align}
	\Ew \sql \sigmant \sqr  = \sum\limits_{n=1}^{N} \sigman \Pr\prl n_t = n \prr 
	\label{eq:sigmaExpectation}
\end{align}
for each $t \in \sql T \sqr$.
If we substitute \eqref{eq:sigmaExpectation} into \eqref{eq:nestedExpectation} and use the result in
\eqref{eq:pluginOptimal}, then we obtain
\begin{align}
	&\Ew \sql \LVert \Gw- \DFw \RVert_2^2 \sqr - \Ew \sql \LVert \Gsw - \DFw \RVert_2^2 \sqr\nn\\
	& = \frac{1}{T^2} \prl  \sum_{n=1}^N \sigman \sum_{t=1}^T  \Pr\prl n_t = n \prr 
	- T\sigmaOpt \prr \Sw, \nn\\
	& =  \frac{1}{T^2} \prl  \sum_{n=1}^N \sigman \sum_{t=1}^T \Ew \sql \1\crl n_t = n \crr \sqr
	- T\sigmaOpt\prr \Sw \label{eq:indicatorProbability} \\
	& = \frac{1}{T^2} \prl  \sum_{n=1}^N \sigman \Ew \sql \gammaTn \sqr
	- \sigmaOpt  \sum_{n=1}^N\Ew \sql \gammaTn \sqr \prr \Sw \label{eq:gammaSum} \\
	& = \frac{1}{T^2} \Ew \sql \sum_{n=1}^N \Delta_n \gammaTn \sqr \Sw, \label{eq:gammaDelta}
\end{align}
where in \eqref{eq:indicatorProbability} we used $\Pr\prl n_t = n \prr = \Ew \sql \1\crl n_t = n \crr \sqr$,
in \eqref{eq:gammaSum} we used $\gammaTn = \sum_{t=1}^{T} \1\crl n_t = n \crr$ 
and $\sum_{n=1}^{N} \gammaTn = T$,
and in \eqref{eq:gammaDelta} we used $\Delta_n = \sigman - \sigmaOpt$.
We note that the term $\Ew \sql \sum_{n=1}^N \Delta_n \gammaTn \sqr \Sw $ is similar to
the pseudo-regret term that appears in stochastic multi-armed bandit problems, where
there are $N$ arms with unknown reward distributions \cite{BubeckC2012}.
We derive the strategy in \eqref{eq:pickOracle} based on similar arguments,
where we leverage a novel application of the Hanson-Wright inequality to
the trace of the sample covariance matrix of each mini-batch oracle to prove concentration inequalities.

To prove an upper bound on \eqref{eq:gammaDelta}, we first show in Lemma~\ref{lemma:quadraticForm} that $\Vtn$ can be written as a quadratic form of sub-Gaussian random variables as
\begin{align*}
	\Vtn = \stn^{T} \Atn \stn, \;\; n\in \sql N \sqr,
\end{align*}
where
$
	\stn \defi \prl \GOnewn^T,\ldots,\Ggammawn^T\prr^T,
$
and
\begin{equation*}
	\Atn = \frac{1}{\gammatn-1} \prl \mI - \frac{1}{\gammatn} \mE \prr,
\end{equation*}
$\mI \in \mathbb{R}^{d \gammatn \times d\gammatn}$ is an identity matrix, and
$\mE \in \mathbb{R}^{d \gammatn \times d\gammatn}$ is a block matrix with $d\times d$ identity blocks.
We next apply the Hanson-Wright inequality \cite{HansonW1971,RudelsonV2013} to $\Vtn$ for each $n\in\sql N\sqr$ to obtain high confidence bounds.
This inequality provides a tail probability bound for an arbitrary quadratic function of independent sub-Gaussian random variables.
We present it in the appendix for completeness.
Moreover, Lemma~\ref{lemma:main}
shows that the tail probability of the trace of the sample covariance matrix of each mini-batch oracle satisfies, for any $\eps>0$,
\begin{align}
	\Pr \prl \Vtn - \sigman \Sw > \eps \prr 
	\leq \exp \prl {- \prl \gammatn -1 \prr \phi \prl \eps \prr} \prr, \label{eq:concentration}
\end{align}
where
\begin{align*}
	\phi \prl \eps \prr 
	&\defi \frac{ c \eps}{\beta P}\min\prl 1,\frac{\eps/d}{\beta P}\prr,
\end{align*}
for each $n\in\sql N\sqr$.
We observe that $\phi = f^{-1}$, where $f$ is defined in \eqref{eq:funcf}.

Note that \eqref{eq:concentration} is equivalent to stating that, for any $\delta \in (0,1)$, 
\begin{align}
	\Vtn - f \prl \frac{1}{\gammatn-1} \ln\prl \frac{1}{\delta}\prr\prr \leq \sigman \Sw
	\label{eq:epsilonDelta}
\end{align}
with probability
at least $1-\delta$. Using this result, we propose the UCB-type strategy in \eqref{eq:pickOracle} to pick the mini-batch oracle on round $t$.
In particular, we show in Lemma~\ref{lemma:regret} 
that, for any $\alpha>2$, we have
\begin{align}
	\Ew \sql \sum_{n=1}^N \Delta_n \gammaTn \sqr \Sw
	 \leq \prl \Cone \ln\prl T \prr + \Ctwo \prr \Sw,
	 \label{eq:UB}
\end{align}
where
\begin{align}
	\Cone 
	\defi\sum_{n:\Delta_n>0} \frac{\alpha \Delta_n }{\phi \prl \Delta_n\Sw/2\prr},
	\;\;\;
	\Ctwo
	\defi \prl \sum_{n=1}^{N} \Delta_n \prr \frac{2\prl \alpha-1\prr}{\alpha-2}.\nn
\end{align}

Finally, if we use \eqref{eq:UB} in \eqref{eq:gammaDelta}, then we obtain
\begin{align}
	\Ew \sql \LVert \Gw- \DFw \RVert_2^2 \sqr - \Ew \sql \LVert \Gsw - \DFw \RVert_2^2 \sqr
	\leq \frac{1}{T^2} \prl \Cone \ln\prl T \prr + \Ctwo \prr \Sw,
	\label{eq:finalUpperBound}
\end{align}
where substituting $\Ew \sql \LVert \Gsw - \DFw \RVert_2^2 \sqr = \sigmaOpt \Sw/T$
in \eqref{eq:finalUpperBound} yields the desired result.
\end{proof}

\section{SGD Performance Under Strongly Convex Objectives}
In this section, we investigate the performance of EE-Grad with strongly convex objective functions with Lipschitz continuous gradients.
That is, we assume that the gradient $\DF$ is Lipschitz continuous with Lipschitz constant $L>0$, i.e.,
\begin{equation*}
	\LVert \DF \prl \bw \prr - \DF \prl \overline{\bw} \prr \RVert_2 
	\leq L \LVert \bw - \overline{\bw} \RVert_2, \;\;\; \forall \bw,\overline{\bw}\in\mathbb{R}^d,
\end{equation*}
and there exists $m>0$ such that
\begin{align*}
	F\prl \overline{\bw} \prr \geq F\prl \bw \prr + \DF\prl \bw \prr^T\prl \overline{\bw} - \bw \prr
+ \frac{1}{2}m \LVert \overline{\bw} - \bw \RVert_2^2,\;\;\;
	\forall\overline{\bw},\bw\in\mathbb{R}^d.
\end{align*}
Let $\bw_*  = \argmin_{\bw \in \mathbb{R}^d} F\prl \bw\prr$ be the global minimizer.
We first describe the {\it optimal} mini-batch SGD algorithm that uses the optimal mini-batch oracle on each iteration.
We next compare its performance to EE-Grad in terms of the rate of convergence to the global solution $\bw_*$.
Note that the cost function $C\prl \theta\prr$, and hence the optimal mini-batch size, is allowed to vary across iterations of the SGD algorithm.
We use the subscript $k$, which denotes the SGD iteration, for the quantities introduced in Section~\ref{sec:problemDescription} and Section~\ref{sec:Algorithm} to emphasize the iteration dependence whenever necessary.

On each iteration $k=1,\ldots,K$, the optimal mini-batch SGD algorithm that knows the optimal mini-batch oracle $\MBOoptk$ distributes the per-round budget $B_k$ to it producing $\Gtok = \GtoknOpt$ on each trial $t=1,\ldots,T$.
After $T$ trials, it computes its final stochastic gradient as
$
	\Gok = (1/T) \sum_{t=1}^T \Gtok,
$
and sets the next iterate as
$
	\bwo_{k+1} = \bwo_k - \eta \Gok.
$
We observe that $\bw_k$ and $\bwo_k$ may be different over iterations, so the true
gradients $\DFak$ and $\DFok$ also may differ.
Also, note that $\Gok$ satisfies
\begin{align*}
	\Ew \sql \LVert \Gok - \DFok \RVert_2^2 \sqr = \frac{\sigmaOpt \Swo}{T} ,
\end{align*}
where $\Swo \defi \Tr \prl \mMwkOpt \prr$ for each $k\in\sql K \sqr$.
In this section, we focus on the case where
\begin{align*}
	\mMw \defi \diag \prl \DFw_1^2,\ldots,\DFw_d^2\prr
\end{align*}
for any $\bw \in \mathbb{R}^d$, which implies that $\Sw = \LVert \DFw \RVert_2^2$.

We define the expected gaps of EE-Grad and of the optimal mini-batch SGD algorithm
with respect to the global minimizer $\bw_*$ on each iteration $k$ as
\begin{align}
	\JAk  &\defi  \E\sql F\prl \bw_k \prr\sqr - \Fstar,\;\;\;\;
	\JOk \defi  \E \sql F\prl \bwo_k \prr\sqr - \Fstar, \label{eq:GapDef}
\end{align}
respectively.
The next theorem shows how these expected gaps evolve over iterations.

\begin{theorem}
\label{thm:SGDapplication}
Suppose that the step size $\eta_k$ is sufficiently small so that it satisfies
\begin{align}
	0 < \eta_k < \frac{2}{L \prl 1 + \ZTwk \prr}.
	\label{eq:assumptionEta}
\end{align}
Then, on each iteration $k$, the expected gap of the optimal mini-batch SGD algorithm satisfies
\begin{align*}
	\JOkp &\leq \taoOpt \JOk,
\end{align*}
where
\begin{align*}
	0 < \taoOpt &\defi mL\eta_k^2 \prl 1 + \sigmaOptk /T \prr - 2m\eta_k +1 < 1.
\end{align*}
Moreover, the expected gap of the EE-Grad Algorithm on iteration $k$ satisfies
\begin{align*}
	\JAkp \leq \taoAlg \JAk,
\end{align*}
where
\begin{align*}
	0 < \taoAlg \defi \taoOpt + mL\eta_k^2 \OTk < 1,
\end{align*}
and $\OTk \defi \ZTwk - \sigmaOptk/T = \Conek \ln T /T^2  + \Ctwok /T^2 > 0$, where
$\OTk\rightarrow 0$ as $T\rightarrow\infty$.
\end{theorem}
\begin{proof}
First note that since $\DF$ is Lipschitz continuous with Lipschitz constant $L>0$, it
satisfies \cite{BottouCN2016}
\begin{align}
	F\prl \bw \prr \leq F\prl \overline{\bw}\prr + \DF\prl \overline{\bw} \prr ^T \prl \bw - \overline{\bw}\prr 
			     + \frac{1}{2}L \LVert \bw - \overline{\bw} \RVert_2^2,
			     \;\;\;\ \forall \bw,\overline{\bw} \in \mathbb{R}^d,
			     \nn
\end{align}
which implies that on each iteration $k$, we have
\begin{align}
	F\prl \bw_{k+1} \prr - F\prl \bw_k \prr 
	\leq -\eta_k \DF\prl \bw_k \prr ^T \Gwk  + \frac{1}{2}L\eta_k^2 \LVert \Gwk \RVert_2^2.
\end{align}
By taking conditional expectations of both sides and rearranging the terms, we obtain
\begin{align}
	\Ek \sql F\prl \bw_{k+1} \prr\sqr - F\prl \bw_k \prr
	\leq -\eta_k \Swk \prl 1 - \frac{1}{2} \eta_k L\prl 1 + \ZTwk \prr\prr.
	\label{eq:iterationAlg}
\end{align}
Performing the same steps on the optimal mini-batch SGD algorithm yields
\begin{align}
	\Ek \sql F\prl \bwo_{k+1} \prr\sqr - F\prl \bwo_k \prr
	\leq -\eta_k \Swo \prl 1 - \frac{1}{2} \eta_k L\prl 1 + \sigmaOptk/T \prr\prr .
	\label{eq:iterationOpt}
\end{align}

Since $F$ is assumed to be $m$-strongly convex, the optimality gap for any $\bw \in \mathbb{R}^d$ satisfies \cite{BottouCN2016}
\begin{align}
	F\prl \bw \prr - F\prl \bw_*\prr  \leq \frac{1}{2m}\LVert \DF\prl \bw\prr \RVert_2^2.
	\label{eq:SCineq}
\end{align}
The assumption in \eqref{eq:assumptionEta} guarantees that
$1 - \eta_k L\prl 1 + \ZTwk \prr/2 > 0$.
Thus, using \eqref{eq:SCineq} in \eqref{eq:iterationAlg}, subtracting $F\prl \bw_*\prr$ on both sides, and rearranging terms give
\begin{align}
	\Ek \sql F\prl \bw_{k+1} \prr\sqr - F\prl \bw_*\prr
	&\leq F\prl \bw_k \prr - F\prl \bw_*\prr
	-\eta_k \Swk \prl 1 - \frac{1}{2} \eta_k L\prl 1 + \ZTwk \prr\prr \nn\\
	&\leq F\prl \bw_k \prr - F\prl \bw_*\prr
	- 2m \eta_k \prl F\prl \bw_k \prr - F\prl \bw_*\prr \prr \prl 1 - \frac{1}{2} \eta_k L\prl 1 + \ZTwk \prr\prr \nn\\
	& = \prl mL\eta_k^2\prl 1 + \ZTwk \prr - 2m\eta +1 \prr \prl F\prl \bw_k \prr - F\prl \bw_*\prr \prr,\nn \\
	& = \taoAlg  \prl F\prl \bw_k \prr - F\prl \bw_*\prr \prr.\nn
\end{align}
Here if we take expectations of both sides and note the definition in \eqref{eq:GapDef}, then we obtain
$
	\JAkp
	\leq \taoAlg \JAk.
$
Similar steps for the optimal mini-batch SGD algorithm imply
$
	\JOkp
	\leq \taoOpt \JOk,
$
where $\taoAlg = \taoOpt + mL\eta_k^2\OTk $, so that $\taoAlg - \taoOpt \rightarrow 0$
as $T\rightarrow\infty$.
\end{proof}

Here, we note that $\taoAlg$ is a quadratic function of $\eta_k$, minimized at $\eta_k = 1/(1+\ZTwk)$,
and $\taoAlg < 1$ for all $\eta_k$ satisfying \eqref{eq:assumptionEta}.
Similarly, $\taoOpt$ is a quadratic function of $\eta_k$, minimized at $\eta_k = 1/(1+\sigmaOptk/T)$,
and $\taoOpt < 1$ for all $\eta_k$ satisfying \eqref{eq:assumptionEta}.
Also, we observe that
\begin{align*}
	\taoAlg = \taoOpt + mL\eta_k^2\OTk > \taoOpt
\end{align*} 
for all $\eta_k>0$, i.e., $\taoAlg$ is uniformly larger than $\taoOpt$, which implies that
the optimal mini-batch SGD algorithm enjoys faster convergence rate than the proposed algorithm.
However, the gap between them is proportional to $\OTk$ for any given step size $\eta_k>0$,
which is the gap between EE-Grad and the optimal mini-batch SGD algorithm, as shown in Theorem~\ref{thm:main}.
Finally, we note that this gap diminishes as the number of trials $T$ increases, at the expense of
larger total incurred cost.
In the next section, we illustrate our theoretical results with numerical examples.

\section{Numerical Results}
In this section, we present a numerical example based on synthetic data to illustrate
our main results.
We consider $d=2$ dimensional case, where the objective function and its gradient are
$F\prl \bw \prr = \bw^T \bw/2$ and $\DFw = \bw$, respectively, where $\Fstar = 0$ with
$\bw_* = \prl 0, 0 \prr^T$.

We assume that $\mMw = \diag \prl w_1^2, w_2^2\prr$, and each stochastic gradient
$\gwt$ with fidelity $\theta>0$ has uncorrelated Gaussian components with the parameters
$w_1^2/\theta$ and $w_2^2/\theta$, respectively.
We next assume that the unknown parameters of the mini-batch oracles are given by
$\sigma_1^2 = 50, \sigma_2^2 = 26, \sigma_3^2 = 16.7$, and run the EE-Grad algorithm
and the mini-batch oracles with a randomly generated initial iterate for $T=50$ trials and $K=5$ iterations by using the constant step size $\eta=0.85$, where we obtain expected results over $2000$ independent realizations.
We plot the resulting expected gaps achieved by EE-Grad  and the mini-batch oracles
in Fig.~\ref{fig:T50}.
We repeat the same procedure for $T=200$ and $T=3000$ and plot the results in Fig.~\ref{fig:T200}
and Fig.~\ref{fig:T3000}, where we note that $\sigma_i^2$ are scaled accordingly, so that the results over different $T$s are comparable.

We observe that for this numerical example, the expected gap achieved by the EE-Grad algorithm
is close to that of the optimal mini-batch oracle, where the performance difference between them shrinks with increasing $T$ at the expense of increased total cost, as we proved in Theorem~\ref{thm:SGDapplication}.

\begin{figure}
    \centering
    \begin{subfigure}[b]{0.32\textwidth}
        \includegraphics[width=\textwidth]{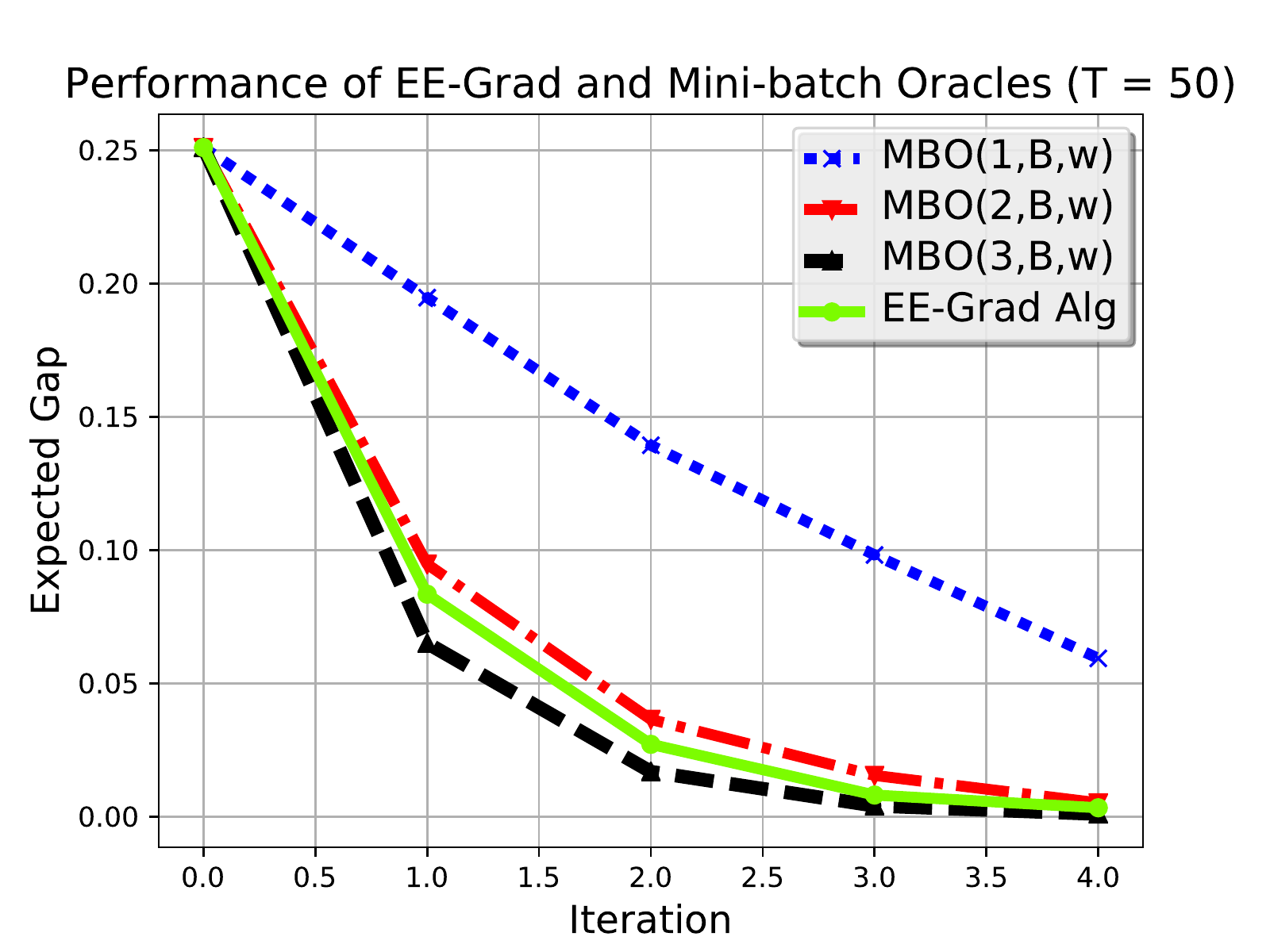}
        \caption{$T=50$.}
        \label{fig:T50}
    \end{subfigure}
    \begin{subfigure}[b]{0.32\textwidth}
        \includegraphics[width=\textwidth]{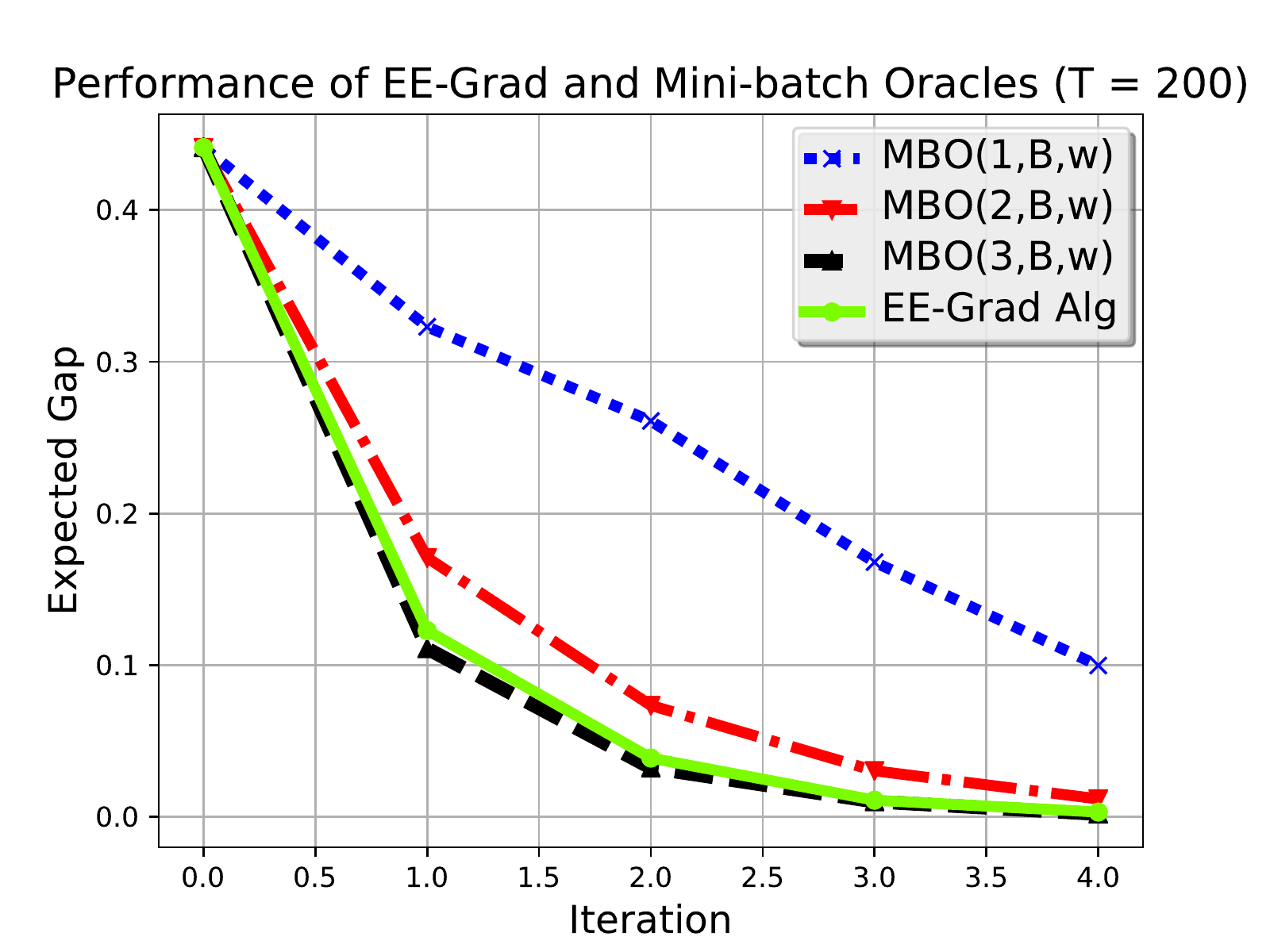}
        \caption{$T=200$.}
        \label{fig:T200}
    \end{subfigure}
    \begin{subfigure}[b]{0.32\textwidth}
        \includegraphics[width=\textwidth]{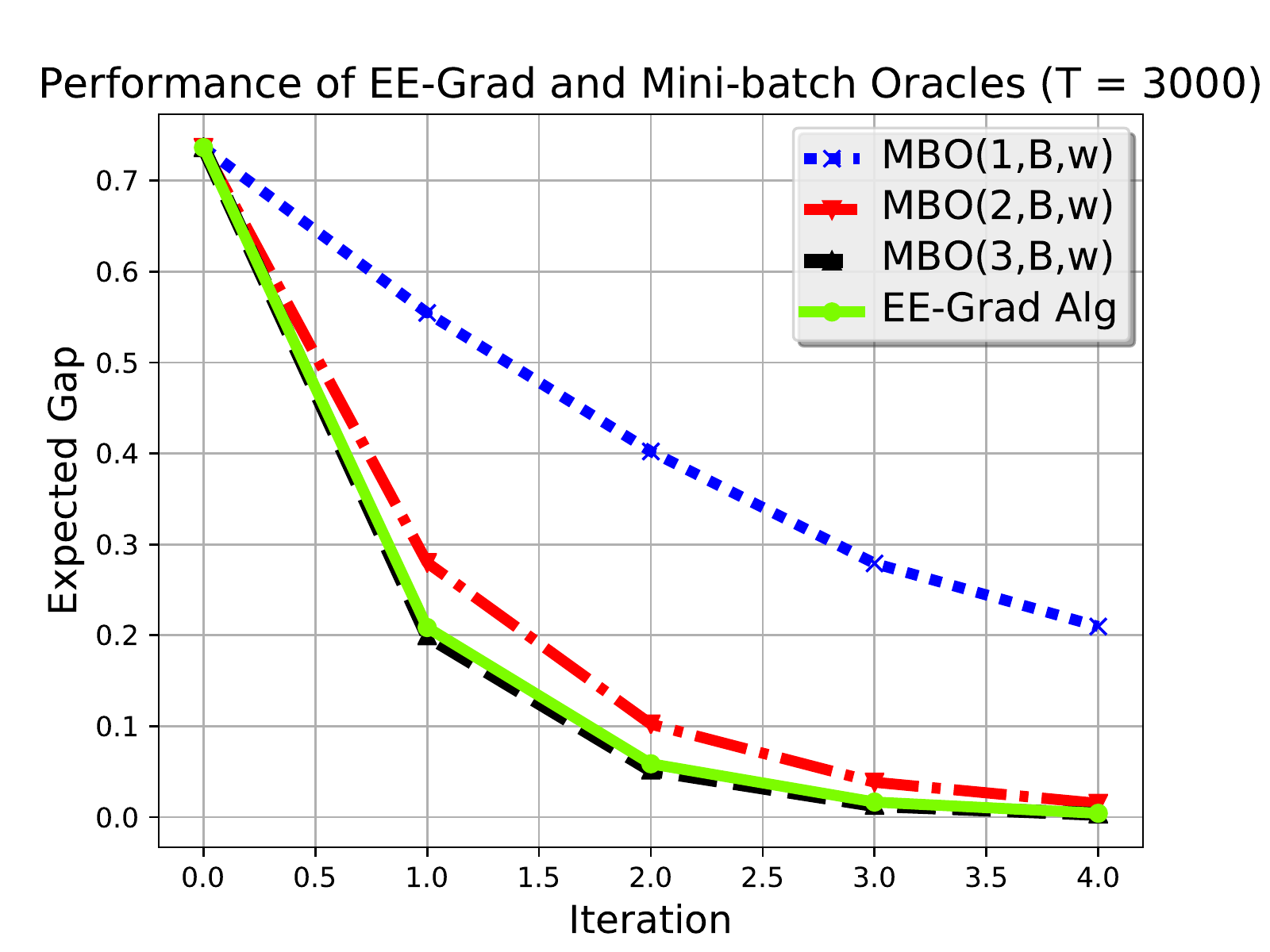}
        \caption{$T=3000$.}
        \label{fig:T3000}
    \end{subfigure}
    \caption{Expected gaps achieved by the EE-Grad Algorithm and the mini-batch oracles for different values of $T$ over $K=5$ iterations.}\label{fig:experiments}
\end{figure}

\section{Discussion}
We presented a new framework to analyze the tradeoff between {\it fidelity} and {\it cost} of computing a stochastic gradient,
where we modeled a noisy gradient as an unbiased estimate of the true gradient such that the noise variance depends on the cost incurred to compute it.
We investigated mini-batch oracles that distribute a limited budget to a mini-batch of
stochastic gradients and averages them to estimate the true gradient, where the averaging operation is also assumed to be costly.
In this framework, the optimal mini-batch size in minimizing the noise variance depends on the underlying cost-fidelity function, which is assumed to be unknown.

We proposed the EE-Grad algorithm that performs sequential trials over different mini-batch oracles to explore
the performance of each mini-batch oracle with high precision and exploit the current knowledge to allocate the budget to the one that seems to provide the best performance.
We demonstrated that the proposed algorithm performs almost as well as the optimal mini-batch oracle on each iteration in expectation.
We next applied this result to the strongly convex objectives with Lipschitz continuous gradients,
and provided a performance guarantee on the rate of convergence with respect to the optimal
mini-batch oracle.
We finally illustrated our theoretical results through numerical experiments on synthetic data.

\newpage 
\bibliographystyle{unsrt} 
\bibliography{abrv,conf_abrv,mad_lib}

\newpage
\appendix
\numberwithin{equation}{section}

\section{Trace of the Sample Covariance Matrix as a Quadratic Form}
\label{app:quadraticForm}
\begin{lemma} 
\label{lemma:quadraticForm}
On each round $t$, the trace of the sample covariance matrix $ \bCtn$ can be written as
\begin{align*}
	\Vtn = \stn^{T} \Atn \stn, \;\; n=1,\ldots,N,
\end{align*}
where
$
	\stn = \prl \GOnewn^T,\ldots,\Ggammawn^T\prr^T
$,
$
	\Atn =\prl \gammatn-1 \prr^{-1}\prl \mI - \gammatn^{-1} \mE \prr,
$
$\mI\in \mathbb{R}^{d \gammatn \times d\gammatn}$ is an identity matrix, and
$\mE \in \mathbb{R}^{d \gammatn \times d\gammatn}$ is a block matrix with $d\times d$ identity blocks.
\end{lemma}
\begin{proof}
Note that
\begin{align*}
	\Vtn  = \Tr \prl \bCtn \prr
	       &= \frac{1}{\gammatn - 1} \prl \sum_{i=1}^{\gammatn}
	       \Giwn^T  \Giwn  - \gammatn  \bmtn ^T\bmtn \prr,
\end{align*}
where
\begin{align*}
	 \bmtn ^T\bmtn
	&= \frac{1}{\gammatn^2} \sum_{i=1}^{\gammatn} 
	\sum_{j=1}^{\gammatn}  \Giwn^T \Gjwn
	= \frac{1}{\gammatn^2} \stn^T \mE \stn.
\end{align*}
Noting $\sum_{i=1}^{\gammatn} \Giwn^T \Giwn = \stn^T \stn$, we conclude that
\begin{align*}
	\Vtn = \frac{1}{\gammatn - 1} \prl \stn^T\stn - \frac{1}{\gammatn} \stn^T\mE\stn \prr
	       =  \stn^T \Atn \stn.
\end{align*}
\end{proof}

\section{Hanson-Wright Inequality}
\label{app:HW}
\begin{lemma}
\label{lemma:HW}
Let $\bW = \sql W_1,\ldots,W_m\sqr^T \in \mathbb{R}^m$, $m>1$, where $W_i$ are zero-mean sub-Gaussian with a parameter $\sigma^2>0$.
Then, given an arbitrary matrix $\mA\in\mathbb{R}^{m\times m}$, we have, for any $\eps>0$, 
\begin{align*}
	\Pr \prl \bW^T \mA \bW - \Ew \sql  \bW^T \mA \bW \sqr > \eps \prr
	\leq \exp \prl -c 
	\min \prl \frac{\eps^2}{\sigma^4\Vert \mA \Vert_{\mathrm{F}}^2},\frac{\eps}{\sigma^2\Vert \mA \Vert} \prr \prr,
\end{align*}
where $\Vert \mA \Vert_{\mathrm{F}}$ and $\Vert \mA \Vert$ are Frobenius and operator norms of $\mA$, and $c>0$ is an absolute constant.
\end{lemma}

\section{Concentration Result on the Trace of the Sample Covariance Matrices}
\label{app:main}
\begin{lemma}
\label{lemma:main}
Suppose that $\gammatn>1$.
Then the tail probability of $\Vtn$ satisfies, for any $\eps>0$,
\begin{align*}
	\Pr \prl \Vtn - \sigman \Sw > \eps \prr 
	\leq \exp \prl {- \prl \gammatn -1 \prr \phi \prl \eps \prr} \prr,
\end{align*}
where
\begin{align*}
	\phi \prl \eps \prr 
	&\defi \frac{ c \eps}{\beta P}\min\prl 1,\frac{\eps/d}{\beta P}\prr,
\end{align*}
for  $n=1,\ldots,N$, where $c>0$ is an absolute constant.
\end{lemma}
\begin{proof}
Note that
$
	\mI - \prl 1/\gammatn \prr \mE
$
is a $d\gammatn \times d\gammatn$ block matrix with $d \times d$ blocks,
where the diagonal and non-diagonal matrices are given by
$\frac{\gammatn - 1}{\gammatn} \mI$ and $-\frac{1}{\gammatn}\mI$, respectively,
and $\LVert \mI \RVert_{\mathrm{F}}^2 = d$.
This implies
\begin{align*}
	\LVert \Atn \RVert_{\mathrm{F}}^2
	& = \frac{1}{\prl \gammatn - 1 \prr^2} 
	\prl \gammatn \prl \frac{ \gammatn-1}{\gammatn}\prr^2 \LVert \mI \RVert_{\mathrm{F}}^2 
	       + \prl \gammatn - 1\prr \gammatn \frac{1}{\gammatn^2} \LVert \mI \RVert_{\mathrm{F}}^2 \prr\nn\\
	 &= \frac{d}{\gammatn - 1}.
\end{align*}

Next suppose that
$\bv = \prl \bv_1^T, \ldots, \bv_{\gammatn}^T \prr^T \in \mathbb{R}^{d\gammatn}$
such that $\bv_i\in\mathbb{R}^d$ and $\Vert \bv \Vert_2 = 1$.
Then we write
\begin{align*}
	\LVert \Atn \bv \RVert_2^2
	&= \frac{1}{\prl \gammatn - 1 \prr^2}
	      \prl \LVert \bv \RVert_2^2 + \frac{1}{\gammatn^2} \LVert \mE\bv \RVert_2^2
	      - \frac{2}{\gammatn} \bv^T \mE \bv \prr\nn\\
	&=  \frac{1}{\prl \gammatn - 1 \prr^2}
	      \prl 1 - \frac{1}{\gammatn} \LVert \sum_{i=1}^{\gammatn} \bv_i \RVert_2^2 \prr
	 \leq \frac{1}{\prl \gammatn - 1\prr^2},
\end{align*}
where equality is achieved by
$\bv = \prl \bv_1^T, \ldots, \bv_{\gammatn}^T \prr^T$
such that $\bv_1 = \prl \frac{1}{\sqrt{2}},0,\ldots,0\prr $, $\bv_2 = -\bv_1$,
and $\bv_i = \prl 0,\ldots,0\prr$ for $i=3,\ldots,\gammatn$.
This yields
\begin{align*}
	\LVert \Atn \RVert
	= \sup_{\LVert \bv \RVert_2 = 1} \LVert \Atn \bv\RVert_2
	=  \prl \gammatn - 1 \prr^{-1}.
\end{align*}

We finally note that the trace of the sample covariance matrix can be written as
\begin{align*}
	\Vtn
	&= \frac{1}{\gammatn - 1} \sum_{i=1}^{\gammatn} \prl \Giwn - \bmtn \prr^T \prl \Giwn - \bmtn \prr\nn\\
	&= \frac{1}{\gammatn - 1} \sum_{i=1}^{\gammatn} \prl \bQ_{i}\prl n\prr - \bqtn \prr^T  \prl \bQ_{i}\prl n\prr - \bqtn \prr,
\end{align*}
where $\bQ_{i}\prl n\prr \defi \Giwn - \DFw$ for $i \in \sql \gammatn \sqr$, and 
$
	\bqtn = \prl 1/ \gammatn \prr \sum_{i=1}^{\gammatn} \bQ_{i}\prl n\prr.
$
This implies the same expression holds for the mean-removed versions of $\Giwn$s.
Hence, we can assume that $\Ew \sql \Giwn \sqr = 0$.
We apply Lemma~\ref{lemma:HW} to $\Vtn$ by using
Lemma~\ref{lemma:quadraticForm} to get, for any $\eps>0$,
\begin{align*}
	\Pr \prl \Vtn - \sigman \Sw > \eps \prr 
	\leq \exp \prl {- \prl \gammatn -1 \prr \phi_n \prl \eps \prr} \prr,
\end{align*}
where
\begin{align*}
	\phi_n \prl \eps \prr 
	&\defi \frac{ c \eps}{\sigman \Sw }\min\prl 1,\frac{\eps/d}{\sigman \Sw}\prr,
\end{align*}
which is strictly increasing in $\eps$, for $n \in \sql N \sqr$, where $c>0$ is an absolute constant.
Finally, we note
$
	\phi_n \prl \eps \prr \geq \phi \prl \eps\prr,
$
since we assumed
$
	\max_{n=1,\ldots,N} \sigman \leq \beta,
$
and $\Sw \leq P$.
This concludes the proof.
\end{proof}

\section{Pseudo-Regret Bound}
\label{app:regret}
\begin{lemma}
\label{lemma:regret}
For any $\alpha>2$, the pseudo-regret term in \eqref{eq:gammaDelta} satisfies, for any $T$,
\begin{align*}
	\Ew \sql \sum_{n=1}^N \Delta_n \gammaTn \sqr \Sw
	 \leq \prl \Cone \ln\prl T \prr + \Ctwo \prr \Sw,
\end{align*}
where
\begin{align}
	\Cone 
	\defi\sum_{n:\Delta_n>0} \frac{\alpha \Delta_n }{\phi \prl \Delta_n\Sw /2\prr},
	\;\;\;
	\Ctwo
	\defi \prl \sum_{n=1}^{N} \Delta_n \prr \frac{2\prl \alpha-1\prr}{\alpha-2}. \label{eq:constants}
\end{align}
\end{lemma}
\begin{proof}
We follow along similar steps to the proof of Theorem 2.1 in \cite{BubeckC2012}.
Suppose that $n_t = n$, and consider the events
\begin{align*}
	E_{t,1} &\defi \crl \VtnOpt - f \prl \frac{\alpha\ln\prl t\prr}{\gammatnOpt-1} \prr \geq \sigmaOpt \Sw \crr,\\
	E_{t,2} &\defi \crl \Vtn <   \sigman\Sw - f \prl \frac{\alpha\ln\prl t\prr}{\gammatn-1} \prr \crr,\\
	E_{t,3} &\defi \crl \gammatn < 1 + \frac{\alpha\ln\prl T \prr}{\phi\prl \Delta_n\Sw/2\prr}\crr. \nn
\end{align*}
We claim that $E_{t,1} \cup E_{t,2} \cup E_{t,3}$ must occur.
Assume, by contradiction, that $E_{t,i}$ are all false.
We obtain
\begin{align}
	\VtnOpt - f \prl \frac{\alpha\ln\prl t\prr}{\gammatnOpt - 1} \prr 
	 < \sigmaOpt \Sw
	 &= \sigman \Sw - \Delta_n \Sw\nn\\
	& \leq \Vtn +  f \prl \frac{\alpha\ln\prl t\prr}{\gammatn - 1} \prr - \Delta_n \Sw \label{eq:contradiction1}.
\end{align}
By assumption $E_{t,3}$ is false, and we have
$
	 \gammatn - 1 \geq \alpha\ln\prl T \prr/\phi\prl \Delta_n \Sw /2\prr,
$
which is equivalent to
\begin{align}
	 \Delta_n\Sw \geq 2f \prl \frac{\alpha\ln\prl T \prr}{\gammatn - 1} \prr, \label{eq:contradiction2}
\end{align}
If we use \eqref{eq:contradiction2} in \eqref{eq:contradiction1},
then we obtain the following result, which contradicts the rule in \eqref{eq:pickOracle}:
\begin{align}
	\VtnOpt - f \prl \frac{\alpha\ln\prl t\prr}{\gammatnOpt - 1} \prr 
       < \Vtn -  f \prl \frac{\alpha\ln\prl t\prr}{\gammatn - 1} \prr. \nn
\end{align}

For all $n$ such that $\Delta_n>0$, we define
\begin{equation*}
	M_n \defi \left \lceil \frac{\alpha \ln\prl T \prr}{\phi\prl \Delta_n\Sw /2\prr} \right \rceil. \nn
\end{equation*}
We next upper bound $\Ew \sql \gammaTn \sqr$ as
\begin{align}
	\Ew \sql \gammaTn \sqr
	&= \Ew \sql \sum_{t=1}^T \1\crl n_t = n \text{ and } \gammatn < M_n \crr \sqr
	+ \Ew \sql \sum_{t=1}^T \1\crl n_t = n \text{ and } \gammatn \geq M_n \crr \sqr \nn\\
	&\leq  M_n
	  +  \Ew \sql \sum_{t=M_n +1}^T \1\crl n_t = n \text{ and } \gammatn \geq M_n \crr \sqr. \label{eq:upperBound1}
\end{align}
In \eqref{eq:upperBound1}, we observe that $\gammatn \geq M_n$ is equivalent to $E_{t,3}$ being false,
which is further equivalent to $E_{t,1} \cup E_{t,2}$ being true, i.e., $E_{t,1}$ or $E_{t,2}$ must occur.
Therefore we can further upper bound \eqref{eq:upperBound1} as
\begin{align}
	 \Ew \sql \gammaTn \sqr
	 &\leq  M_n
	  +  \Ew \sql \sum_{t=M_n+1}^T \1\crl E_{t,1} \text{ or } E_{t,2} \text{ is true} \crr \sqr\nn\\
	 &=  M_n
	  + \sum_{t=M_n+1}^T \Pr \prl  E_{t,1} \cup E_{t,2} \text{ is true} \prr \nn \\
	  &\leq M_n
	  +   \sum_{t=M_n+1}^T \Pr \prl  E_{t,1} \prr  + \sum_{t=M_n+1}^T \Pr \prl E_{t,2} \prr.
	  \label{eq:upperBound2}
\end{align}
where we used the union bound.
We upper bound $\Pr \prl E_{t,1}\prr$ for each $t=M_n+1,\ldots,T$.
Note that
\begin{align}
	\Pr \prl  E_{t,1} = 1 \prr
	= \Pr \prl  \VtnOpt - f \prl \frac{\alpha\ln\prl t\prr}{\gammatnOpt - 1} \prr \geq \sigmaOpt \Sw\prr,
	\label{eq:A1}
\end{align}
where $\gammatnOpt$ can take values in $\crl 2,\ldots,t \crr$.
Hence we apply the union bound in \eqref{eq:A1}, which yields
\begin{align}
	\Pr \prl  E_{t,1} = 1 \prr
	&\leq \sum_{s=1}^t \Pr \prl  \Vsn - f \prl \frac{\alpha\ln\prl t\prr}{s} \prr \geq \sigmaOpt \Sw \prr
	\leq \sum_{s=1}^t \frac{1}{t^{\alpha}} \label{eq:upperBound3}
	 = t^{1-\alpha},
\end{align}
where \eqref{eq:upperBound3} follows from \eqref{eq:epsilonDelta}.
Here, $\Vsn$ is the trace of a sample covariance matrix given $s+2$ independent random vectors with
sub-Gaussian components with the parameter $\sigmaOpt \Sw$.
Hence we obtain
\begin{align}
	\sum_{t=M_n+1}^T \Pr \prl  E_{t,1} = 1 \prr
	\leq  \sum_{t=M_n+1}^T t^{1-\alpha}
	\leq  \sum_{t=1}^{\infty} t^{1-\alpha}
	\leq 1 + \int_{1}^{\infty} t^{1-\alpha} dt
	= \frac{\alpha-1}{\alpha-2}.  \label{eq:upperBound4}
\end{align}
The same upper bound holds for $\Pr \prl  E_{t,2} = 1 \prr$ so that
\begin{align*}
	\sum_{t=M_n+1}^T \Pr \prl  E_{t,2} = 1 \prr
	\leq \frac{\alpha-1}{\alpha-2}.
\end{align*}
By incorporating these upper bounds
into \eqref{eq:upperBound2} we obtain
\begin{align*}
	\Ew \sql \gammatn \sqr
	\leq M_n + 2 \prl \alpha -1 \prr /\prl \alpha-2 \prr.
\end{align*}
Finally we use this result to get
\begin{align}
	\Ew \sql \sum_{n=1}^N \Delta_n \gammaTn \sqr \Sw
	 \leq \prl \Cone \ln\prl T \prr + \Ctwo \prr \Sw, \nn
 \end{align}
 where $\Cone$ and $\Ctwo$ are defined in \eqref{eq:constants}.
 \end{proof}
 
\end{document}